\newif\ifarxiv
\arxivtrue

\documentclass{article}
\usepackage{graphicx} % Required for inserting images
\usepackage[utf8]{inputenc}
\usepackage{geometry}
\usepackage{amsmath, amssymb, amsthm, bbm,  mathtools, natbib}
\allowdisplaybreaks
\usepackage[dvipsnames]{xcolor}
\usepackage{graphicx} % Required for inserting images
\usepackage{hyperref}
\usepackage{xcolor}
\usepackage{fullpage}
\usepackage{authblk}

\newcommand{\1}{\mathbbm{1}}
\newcommand{\cC}{\mathcal{C}}

\DeclareMathOperator{\argmin}{\text{argmin}}

\newtheorem{lemma}{Lemma}
\newtheorem{theorem}{Theorem}

\newenvironment{proofof}[1]{\par{\noindent \textit{Proof of #1.}}}{\hspace*{\fill} $\qed$ \par}

\usepackage[ruled,vlined]{algorithm2e}

\newcommand{\osa}{\textsf{One-Step-Ahead}}
\newcommand{\aosa}{\textsf{Almost-One-Step-Ahead}}
\newcommand{\ECE}{\mathrm{ECE}}
\newcommand{\CalDist}{\mathrm{CalDist}}

\title{\fontsize{15}{8} An Elementary Predictor Obtaining $2\sqrt{T} + 1$ Distance to Calibration}
\ifarxiv
\author{Eshwar Ram Arunachaleswaran}
\author{Natalie Collina} 
\author{Aaron Roth} 
\author{Mirah Shi}
\affil{Department of Computer and Information Sciences, University of Pennsylvania}
\else 
\author{}
\fi
\begin{document}

\maketitle

\begin{abstract}
   \cite{blasiok2023unifying} proposed \emph{distance to calibration} as a natural measure of calibration error that unlike expected calibration error (ECE) is continuous. Recently, \cite{qiao2024distance} (COLT 2024) gave a non-constructive argument establishing the existence of a randomized online predictor that can obtain $O(\sqrt{T})$ distance to calibration in expectation in the adversarial setting, which is known to be impossible for ECE. They leave as an open problem finding an explicit, efficient, deterministic  algorithm. We resolve this problem and give an extremely simple, efficient, deterministic algorithm that obtains distance to calibration error at most $2\sqrt{T} + 1$.
\end{abstract}
\setcounter{page}{0}
\thispagestyle{empty}
 \clearpage

\section{Introduction}
Probabilistic predictions of binary outcomes are said to be \emph{calibrated}, if, informally, they are unbiased conditional on their own predictions. For predictors that are not perfectly calibrated, there are a variety of ways to measure calibration error. Perhaps the most popular measure is Expected Calibration Error (ECE), which measures the average bias of the predictions, weighted by the frequency of the predictions. ECE has a number of difficulties as a measure of calibration, not least of which is that it is discontinuous in the predictions. Motivated by this, \cite{blasiok2023unifying} propose a different measure: distance to calibration, which measures how far a predictor is in $\ell_1$ distance from the nearest perfectly calibrated predictor. In the online adversarial setting, it has been known since \cite{foster1998asymptotic} how to make predictions with ECE growing at a rate of $O(T^{2/3})$. \cite{qiao2021stronger} show that obtaining $O(\sqrt{T})$ rates for ECE is impossible. Recently, in a COLT 2024 paper, \cite{qiao2024distance} showed that it was possible to make sequential predictions against an adversary guaranteeing expected distance to calibration growing at a rate of $O(\sqrt{T})$. Their algorithm is the solution to a minimax problem of size doubly-exponential in $T$. They leave as an open problem finding an explicit, efficient, deterministic algorithm for this problem. In this paper we resolve this problem, by giving an extremely simple such algorithm with an elementary analysis.

\begin{algorithm}[H]
 \KwIn{Sequence of outcomes $y^{1:T} \in \{0,1\}^T$}
  \KwOut{Sequence of predictions $p^{1:T} \in \{0, \frac{1}{m},..., 1\}^T$ for some discretization parameter $m>0$}
 %\textbf{Sequence of hypothetical predictions $\{\Tilde{p}^1, \Tilde{p}^2, ..., \Tilde{p}^T\}$}
 %initialization\;
 \For{$t = 1$ \KwTo $T$}{   
    %\If{$t = 1$}{
     %   Choose any $p \in \{0, \frac{1}{m},..., 1\}$ and predict $p^1 = p$\;
      %  \esh{Should we instead add a case defining the biases to be zero for a prediction if it has never appeared?}\ar{Why do we need a special case at all? When all biases are 0 the main case still applies no?}
    %}
    \hspace{.2em} Given look-ahead predictions $\tilde p^{1:t-1}$, define the look-ahead bias conditional on a prediction $p$ as:
    $$\alpha_{\tilde p^{1:t-1}}(p) := \sum_{s=1}^{t-1} \1[\tilde p^s = p] (\tilde p^s - y^s)$$
    Choose two adjacent points $p_i = \frac{i}{m}, p_{i+1} = \frac{i+1}{m}$ satisfying: $$\alpha_{\Tilde{p}^{1:t-1}}(p_i) \leq 0 \text{ and } \alpha_{\Tilde{p}^{1:t-1}}(p_{i+1}) \geq 0$$
    Arbitrarily predict ${p}^t = p_i$ or ${p}^t = p_{i+1}$\;
    \hspace{.2em} Upon observing the (adversarially chosen) outcome $y^t$, set look-ahead prediction $$\Tilde{p}^t = \argmin_{p \in \{p_i, p_{i+1}\}} |p - y^t|$$
 }
 \caption{$\aosa$}
 \label{alg:1}
\end{algorithm}

\section{Setting}
We study a sequential binary prediction setting: at every round $t$, a forecaster makes a prediction $p^t \in [0,1]$, after which an adversary reveals an outcome $y^t \in \{0,1\}$. Given a sequence of predictions $p^{1:T}$ and outcomes $y^{1:T}$, we measure expected calibration error (ECE) as follows:
\[
\ECE(p^{1:T}, y^{1:T}) = \sum_{p \in [0,1]} \left| \sum_{t=1}^T \1[p^t = p] (p^t - y^t) \right|
\]
Following \cite{qiao2024distance}, we define \textit{distance to calibration} to be the minimum $\ell_1$ distance between a sequence of predictions produced by a forecaster and any \textit{perfectly calibrated} sequence of predictions:
\[
\CalDist(p^{1:T}, y^{1:T}) = \min_{q^{1:T} \in \cC(y^{1:T})} \|p^{1:T} - q^{1:T}\|_1
\]
where $\cC(y^{1:T}) = \{ q^{1:T} : \ECE(q^{1:T}, y^{1:T}) = 0 \}$ is the set of predictions that are perfectly calibrated against outcomes $y^{1:T}$. First we observe that distance to calibration is upper bounded by ECE.

\begin{lemma}[\cite{qiao2024distance}]\label{lem:caldist} 
    Fix a sequence of predictions $p^{1:T}$ and outcomes $y^{1:T}$. Then, $\CalDist(p^{1:T}, y^{1:T}) \leq \ECE(p^{1:T}, y^{1:T})$.
\end{lemma}
\begin{proof}
    For any prediction $p \in [0,1]$, define $$\overline{y}^T(p) = \sum_{t=1}^T \frac{\1[p^t = p]}{\sum_{t=1}^T \1[p^t = p]} y^t$$ to be the average outcome conditioned on the prediction $p$. Consider the sequence $q^{1:T}$ where $q^t = \overline{y}^T(p^t)$. Observe that $q^{1:T}$ is perfectly calibrated. Thus, we have that
    \begin{align*}
        \CalDist(p^{1:T}, y^{1:T}) &\leq \|p^{1:T} - q^{1:T}\|_1 \\
        &= \sum_{t=1}^T |p^t - q^t| \\
        &= \sum_{p\in[0,1]} \sum_{t=1}^T \1[p^t = p] |p - \overline{y}^T(p)| \\
        &= \sum_{p\in[0,1]} |p - \overline{y}^T(p)| \sum_{t=1}^T \1[p^t = p] \\
        &= \sum_{p\in[0,1]} \left|p \sum_{t=1}^T \1[p^t = p] - \overline{y}^T(p) \sum_{t=1}^T \1[p^t = p]\right| \\
        &= \sum_{p\in[0,1]} \left| \sum_{t=1}^T \1[p^t = p] (p - y^t) \right| \\
        &= \ECE(p^{1:T}, y^{1:T})
    \end{align*}  
\end{proof}

The upper bound is not tight, however. The best known sequential prediction algorithm obtains ECE bounded by  $O(T^{2/3})$ \citep{foster1998asymptotic}, and it is known that there is no algorithm guaranteeing ECE below $O(T^{0.54389})$ \citep{qiao2021stronger, dagan2024improved}. \cite{qiao2024distance} give an algorithm that is the solution to a game of size doubly-exponential in $T$ that obtains expected distance to calibration $O(\sqrt{T})$. Here we give an elementary analysis of a simple efficient deterministic algorithm (Algorithm \ref{alg:1}) that obtains distance to calibration $2\sqrt{T} + 1$.

\begin{theorem}\label{thm:alg}
    Algorithm~\ref{alg:1} ($\aosa$) guarantees that against any sequence of outcomes, $\CalDist(p^{1:T}, y^{1:T})\leq 2\sqrt{T} + 1$. 
\end{theorem}

\section{Analysis of Algorithm \ref{alg:1}}

Before describing the algorithm, we introduce some notation. We will make predictions that belong to a grid. Let $B_m = \{0, 1/m,..., 1\}$ denote a discretization of the prediction space with discretization parameter $m > 0$, and let $p_i = i/m$. For a sequence of predictions $\tilde p^1,...,\tilde p^t$ and outcomes $y^1,...,y^t$, we define the bias conditional on a prediction $p$ as:
\[
\alpha_{\tilde p^{1:t}}(p) = \sum_{s=1}^{t} \1[\tilde p^s = p] (\tilde p^s - y^s)
\]

To understand our algorithm, it will be helpful to first state and analyze a hypothetical ``lookahead" algorithm that we call ``$\osa$'', which is closely related to the algorithm and analysis given by \cite{gupta2022faster} in a different model. $\osa$ produces predictions $\Tilde{p}^1,...,\Tilde{p}^T$ as follows. At round $t$, before observing $y^t$, the algorithm fixes two predictions $p_{i}, p_{i+1}$ satisfying $\alpha_{\Tilde{p}^{1:t-1}}(p_i) \leq 0$ and $\alpha_{\Tilde{p}^{1:t-1}}(p_{i+1}) \geq 0$. Such a pair is guaranteed to exist, because by construction, it must be that for any history, $\alpha_{\Tilde{p}^{1:t-1}}(0) \leq 0$ and $\alpha_{\Tilde{p}^{1:t-1}}(1) \geq 0$. Note that a well known randomized algorithm obtaining diminishing ECE (and smooth calibration error) uses the same observation to carefully \emph{randomize} between two such adjacent predictions \citep{foster1999proof,foster2018smooth}.
Upon observing the outcome $y^t$, the algorithm outputs prediction $\Tilde{p}^t = \argmin_{p \in \{p_i, p_{i+1}\}} |p - y^t|$. Naturally, we cannot implement this algorithm, as it chooses its prediction only after observing the outcome, but our analysis will rely on a key property this algorithm maintains---namely, that it always produces a sequence of predictions with ECE upper bounded by $m+1$, the number of elements in the discretized prediction space. 

\begin{theorem}\label{thm:lookahead}
    For any sequence of outcomes, $\osa$ achieves $\ECE(\Tilde{p}^{1:T}, y^{1:T}) \leq m+1$. 
\end{theorem}
\begin{proof}
    We will show that for any $p_i \in B_m$, we have $|\alpha_{\Tilde{p}^{1:T}}(p_i)| \leq 1$, after which the bound on ECE will follow: $\ECE(\Tilde{p}^{1:T}, y^{1:T}) = \sum_{p_i \in B_m} |\alpha_{\Tilde{p}^{1:T}}(p_i)| \leq m+1$. We proceed via an inductive argument. Fix a prediction $p_i \in B_m$. At the first round $t_1$ in which $p_i$ is output by the algorithm, we have that $|\alpha_{\Tilde{p}^{1:t_1}}(p_i)| = |p^{t_1} - y^{t_1}| \leq 1$. Now suppose after round $t-1$, we satisfy $|\alpha_{\Tilde{p}^{1:t-1}}(p_i)| \leq 1$. If $p_i$ is the prediction made at round $t$, it must be that either: $\alpha_{\Tilde{p}^{1:t-1}}(p_i) \leq 0$ and $p_i - y^t \geq 0$; or $\alpha_{\Tilde{p}^{1:t-1}}(p_i) \geq 0$ and $p_i - y^t \leq 0$. Thus, since $\alpha_{\Tilde{p}^{1:t-1}}(p_i)$ and $p_i - y^t$ either take value 0 or differ in sign, we can conclude that
    \[
    |\alpha_{\Tilde{p}^{1:t}}(p_i)| = |\alpha_{\Tilde{p}^{1:t-1}}(p_i) + p_i - y^t| \leq \max\{ |\alpha_{\Tilde{p}^{1:t-1}}(p_i)|, |p_i - y^t| \} \leq 1
    \]
    which proves the theorem.
\end{proof}
Algorithm \ref{alg:1} ($\aosa$) maintains the same state $\alpha_{\tilde p^{1:t}}(p)$ as \osa ~(which it can compute at round $t$ after observing the outcome $y_{t-1}$). In particular, it does not keep track of the bias of its own predictions, but rather keeps track of the bias of the predictions that $\osa$ \emph{would have made}. Thus it can determine the pair $p_i, p_{i+1}$ that \osa~would commit to predict at round $t$. It cannot make the same prediction as $\osa$ (as it must fix its prediction before the label is observed) --- so instead it deterministically predicts $p^t = p_i$ (or $p^t = p_{i+1}$ --- the choice can be arbitrary and does not affect the analysis). Since we have that $|p_i - p_{i+1}| \leq \frac{1}{m}$, it must be that for whichever choice $\osa$ would have made, we have $|\tilde p^t - p^t| \leq \frac{1}{m}$. In other words, although $\aosa$ does not make the same predictions as $\osa$, it makes predictions that are within $\ell_1$ distance $T/m$ after $T$ rounds. The analysis then follows by the ECE bound of $\osa$, the triangle inequality, and choosing $m = \sqrt{T}$.

%We construct a nearby algorithm that maintains the lookahead algorithm, but now---without the ability to look ahead---arbitrarily outputs one of two adjacent predictions at every round. More formally, the algorithm proceeds as follows. At every round $t$, the algorithm finds two predictions $p_{i}, p_{i+1}$ satisfying $\alpha_{\Tilde{p}^{1:t-1}}(p_i) \leq 0$ and $\alpha_{\Tilde{p}^{1:t-1}}(p_{i+1}) \geq 0$ (at round 1, choose any $p_{i}, p_{i+1}$), and arbitrarily\footnote{In particular, this choice could be fully deterministic. Thus, while there is no deterministic algorithm with sublinear calibration error, we in fact are showing a deterministic algorithm with $\sqrt{T}$ calibration distance.} predicts $p^t = p_i$ or $p^t = p_{i+1}$. Upon observing the outcome $y^t$, set $\Tilde{p}^t = p_i$ if $y^t = 0$ and $\Tilde{p}^t = p_{i+1}$ if $y^t = 1$. 

%We remark that, critically, the choice of predictions $p_i,p_{i+1}$ considered by our algorithm is based upon the (hypothetical) execution of the lookahead algorithm. In particular, the quantities $\alpha_{\Tilde{p}^{1:t-1}}(p_i)$ and $\alpha_{\Tilde{p}^{1:t-1}}(p_{i+1})$ are calculated based upon the historical sequence $\Tilde{p}^{1:t-1}$ of the latter's predictions, which is visible to our algorithm at the start of the $t$-th round. 

%Now we show that this algorithm achieves $O(\sqrt{T})$ calibration distance.  
%\vspace{0.5em}
\begin{proofof}{Theorem \ref{thm:alg}}
    Observe that internally, Algorithm \ref{alg:1} maintains the sequence $\Tilde{p}^1,...,\Tilde{p}^t$ which corresponds exactly to predictions made by $\osa$. Thus, by Lemma \ref{lem:caldist} and Theorem \ref{thm:lookahead}, we have that $\CalDist(\Tilde{p}^{1:T}, y^{1:T})\leq \ECE(\Tilde{p}^{1:T}, y^{1:T}) \leq m+1$. Then, we can compute the distance to calibration of the sequence $p^1,...,p^T$:
    \begin{align*}
        \CalDist(p^{1:T}, y^{1:T}) &= \min_{q^{1:T} \in \cC(y^{1:T})} \|p^{1:T} - q^{1:T}\|_1 \\
        &= \min_{q^{1:T} \in \cC(y^{1:T})} \|p^{1:T} - \Tilde{p}^{1:T} + \Tilde{p}^{1:T} - q^{1:T}\|_1 \\
        &\leq \|p^{1:T} - \Tilde{p}^{1:T}\|_1 + \min_{q^{1:T} \in \cC(y^{1:T})} \|\Tilde{p}^{1:T} - q^{1:T}\|_1 \\
        &\leq \frac{T}{m} + m + 1
    \end{align*}
    where in the last step we use the fact that $|p^t - \Tilde{p}^t| \leq 1/m$ for all $t$ and thus $\|p^{1:T} - \Tilde{p}^{1:T}\|_1 \leq T/m$. The result then follows by setting $m = \sqrt{T}$.
\end{proofof}

\ifarxiv
\subsection*{Acknowledgements}
 This work was supported in part by the Simons Collaboration on the Theory of Algorithmic Fairness, NSF grants FAI-2147212 and CCF-2217062, an AWS AI Gift for Research on Trustworthy AI, and the Hans Sigrist Prize.
\fi

\bibliography{main}
\bibliographystyle{plainnat}

\end{document}